\PassOptionsToPackage{numbers}{natbib}
\documentclass{article}
\pdfoutput=1 

\usepackage[preprint]{neurips_2020} 

\usepackage[utf8]{inputenc} 
\usepackage[T1]{fontenc}    
\usepackage[]{hyperref}       

\usepackage{url}            
\usepackage{booktabs}       
\usepackage{amsfonts}       
\usepackage{nicefrac}       
\usepackage{microtype}      

\usepackage{amsmath,amsthm,amsfonts,amssymb,mathrsfs,bm,graphicx,amscd,epsfig,psfrag,verbatim,hyperref}
\usepackage[usenames]{color}
\usepackage{fancyhdr}
\usepackage{myAlgorithm}
\usepackage{subfigure}
\usepackage{footmisc}

\newtheorem{theorem}{Theorem}[section]
\newtheorem{lemma}[theorem]{Lemma}

\newtheorem{example}{Example}[section]
\newtheorem{definition}{Definition}

\def\twofig{0.48\textwidth}
\def\twofigminus{0.48\textwidth}
\def\twofigMINUS{0.355\textwidth}

\def\C{\mathbb{C}}

\def\E{\mathbb{E}}
\def\P{\mathbb{P}}

\def\eps{\epsilon}

\def\L{\Lambda}
\def\la{\lambda}

\def\el{\mathcal{L}}

\def\twofig{0.48\textwidth}

\def\KK{\mathbf{K}}
\def\HH{\mathbf{H}}
\def\II{\mathbf{I}}
\def\BB{\mathbf{B}}
\def\LL{\mathbf{L}}
\def\DD{\mathbf{\Delta}}
\def\D{\mathbf{D}}
\def\PP{\mathbf{P}}
\def\QQ{\mathbf{Q}}
\def\UU{\mathbf{U}}
\def\VV{\mathbf{V}}
\def\SSigma{\mathbf{\Sigma}}

\def\d{\mathrm{d}}
\def\det{\mathrm{Det}}

\renewcommand{\l}[0]{\left }
\renewcommand{\r}[0]{\right}

\makeatletter
\renewcommand*{\@cite@ofmt}{\hbox}
\makeatother

\title{Learning from DPPs via Sampling: \\ Beyond HKPV and symmetry\thanks{Authors listed in alphabetical order}}

\author{%
R\'emi Bardenet\\
Universit\'e de Lille, CNRS, Centrale Lille\\
UMR 9189 - CRIStAL, Villeneuve d’Ascq, France \\
\texttt{remi.bardenet@gmail.com} \\
\And
Subhroshekhar Ghosh \\
National University of Singapore , Dept of Math\\
10 Lower Kent Ridge Road, Singapore 119076 \\
\texttt{subhrowork@gmail.com} \\
}

\begin{document}

\maketitle

\begin{abstract}
Determinantal point processes (DPPs) have become a significant tool for recommendation systems, feature selection, or summary extraction, harnessing the intrinsic ability of these probabilistic models to facilitate sample diversity.
The ability to sample from DPPs is paramount to the empirical investigation of these models. Most exact samplers are variants of a spectral meta-algorithm due to Hough, Krishnapur, Peres and Vir\'ag (henceforth HKPV, \cite{HKPV06}), which is in general time and resource intensive. For DPPs with symmetric kernels, scalable HKPV samplers have been proposed that either first downsample the ground set of items, or force the kernel to be low-rank, using e.g. Nyström-type decompositions.

In the present work, we contribute a radically different approach than HKPV.
Exploiting the fact that many statistical and learning objectives can be effectively accomplished by only sampling certain key observables of a DPP (so-called \emph{linear statistics}), we invoke an expression for the Laplace transform of such an observable as a single determinant, which holds in complete generality.
Combining traditional low-rank approximation techniques with Laplace inversion algorithms from numerical analysis, we show how to directly approximate the distribution function of a linear statistic of a DPP. This distribution function can then be used in hypothesis testing or to actually sample the linear statistic, as per requirement. Our approach is scalable and applies to very general DPPs, beyond traditional symmetric kernels.

\end{abstract}

\section{Introduction}
Determinantal point processes (\textit{abbrv.} DPPs) have recently emerged as a powerful modelling paradigm in machine learning. DPPs were first formalized by Macchi \cite{Mac}, to model fermion beams in quantum optics. Subsequently, such a determinantal structure was discovered in many fundamental settings in statistical physics and probability, including, in particular, important models of random matrix theory and associated particle systems. Viewed as a model for generating random subsets of items, DPPs can in particular encode repulsive interaction between these items through a so-called \emph{kernel} matrix. Moreover, inference and sampling can be done in polynomial time \cite{KuTa12}. When the task at hand can be abstracted as selecting a small and diverse set of items of a large universal set, DPPs thus appear as a natural tool. In machine learning, DPPs have been used in pose estimation in videos \cite{KuTa10}, recommendation systems \cite{GBDK19}, text summarization \cite{KuTa12}, coreset construction \cite{TrBaAm19}, feature selection \cite{BeBaCh18Sub}, etc. In all these applications, being able to sample from the learned DPP is essential.

Except for a few specialised kernels (e.g., uniform spanning trees \cite{PrWi98JoA}), the default exact sampler is a spectral meta-algorithm due to Hough, Krishnapur, Peres and Virag (\textit{abbrv.} HKPV, \cite{HKPV06}).
Sampling from DPPs presents its own challenges, pertaining to the complicated algebraic structure inherent in the model, which limits its tractability as a probabilistic object. In particular, the \textit{ambient dimension} as well as the \textit{inherent dimension} of the model (which pertain to the sizes of the universal set and the randomly selected subset, respectively) are usually very large in ML applications. This renders spectral methods such as HKPV, that involve cubic cost manipulations of the DPP kernel, expensive both in terms of time and resources. This has led to a vast body of work on scalable DPP sampling, among which scalable approaches to HKPV through either low-rank approximations of the kernel \cite{KuTa11, GiKuTa12,AKFT13,AFAT13}, or by carefully downsampling the universal set \cite{Der19,DeCaVa19}.

We first note that the practitioner may not be really interested in generating samples of the full random subset $X$ from the DPP as such, but only in obtaining samples of certain important \emph{linear statistics} $\Lambda(\Psi):=\sum_{x\in X} \Psi(x)$, for some complex-valued function $\Psi$ over the universal set. A first example arises when a DPP is used to subsample a large dataset $\{x_1,\dots,x_N\}\subset\mathcal{X}$ into a \textit{coreset} \cite[Section 2.2]{TrBaAm19} for a given loss function $L:\mathcal{X}\times\Theta \rightarrow \mathbb{R}$. This means that we look for a small subset $X$ of $\{x_1,\dots,x_N\}$ and a set of weights $\omega_x, x\in X,$ such that the weighted average of $L(\cdot,\theta)$ over $X$ is close in relative error to the average loss over the whole dataset, either uniformly in $\theta$ or for some fixed value of $\theta$. Once the DPP is fixed, one is thus interested in the distribution of the linear statistics $\Lambda(\omega_{\cdot} L(\cdot,\theta)) = \sum_{x\in X}\omega_x L(x,\theta)$, one statistic per value of $\theta$ considered. A related case of interest is the use of DPPs to select mini-batches in stochastic gradient algorithms \cite{ZhKjMa17}: there again, one is not really interested in the DPP itself, but in the realization of the noisy gradient, another example of linear statistic. Another use case for sampling a linear statistic, and \emph{a fortiori} knowing the distribution of that statistic, is to explore a DPP model. In text summarization \cite{KuTa12} or recommendation systems \cite{GBDK19}, once the kernel is learned in some nonparametric way, one may understand the model by looking at the distribution of linear statistics such as, respectively, the number of characters in a DPP summary, or the the total price of a DPP basket. Finally, in a hypothesis testing setup, it is usually very difficult to compare distributions on subsets of a very large universal set, and it is natural that  effective tests of hypothesis be based on comparing the values of a summary statistic against a threshold. Further, the determination of such thresholds involves estimating only some particular quantiles of the distribution of the relevant summary statistic. In all these, sampling from the corresponding DPP is only a means to obtain a sample of a statistic, which often turns out to be a linear statistic.

In this paper, we investigate a way to directly approximate the distribution function of a given linear statistic of a DPP, and approximately sampling the linear statistic if desired, without ever sampling the underlying DPP.
After introducing DPPs and HKPV sampling in Section~\ref{s:determinantal} and Laplace transforms in Section~\ref{s:laplace}, we contribute in Theorem~\ref{thm:Laplace} an expression for the Laplace transform of a linear statistic of a DPP, in terms of finite \textit{Fredholm determinants}. Our result extends the classical reference \cite{ST} by removing all assumptions on $\KK$ for finite DPPs; in particular, it is the first to encompass \emph{attractive-repulsive} non-symmetric DPPs \cite{GBDK19}. 
In Section~\ref{s:algorithm}, drawing on an extensive repertoire of numerical methods --~to compute the determinants on one hand, and to invert the Laplace transform a on the other~-- we put forward a methodology to approximate the cumulative distribution function (CDF) of nonnegative linear statistics of DPPs. Sampling is then straightforward using the inverse CDF approach \cite{RoCa04}. In Section~\ref{s:experiments}, we numerically investigate our approach, and we demonstrate that it outperforms the natural alternative of first generating a random subset from the DPP and then computing its corresponding linear statistic.
Finally, in Section~\ref{s:discussion}, we discuss possible extensions.

\section{Determinantal point processes and their sampling}
\label{s:determinantal}
A DPP is a probabilistic model for selecting a random subset $X$ of a bigger universal (or ground) set $\Xi=[N]:=\{1,\dots,N\}$, parametrized by an $N\times N$ matrix $\KK$.

\begin{definition}[DPP]
  Let $\KK$ be an $N\times N$ complex matrix. We say that $X\sim\text{\emph{DPP}}(\KK)$ if
  \begin{equation} \label{eq:DPP-def}
  \P(A \subseteq X) = \det[\KK_A], \quad \forall A\subseteq \Xi,
  \end{equation}
where $\KK_A$ is the submatrix of $\KK$ corresponding to the rows and columns of $\KK$ indexed by $A$.
\end{definition}
Conditions must be put on the kernel matrix $\KK$ to ensure that such a probability exists. For instance, when $\KK$ is Hermitian with eigenvalues in the interval $[0,1]$, existence follows from a classical theorem due to Macchi and Soshnikov \cite{Mac,Sos2000}.
Alternately, if $\KK$ has all its eigenvalues in the set $[0,1)$, existence is equivalent to $\mathbf{L}=(\II-\KK)^{-1}\KK$ having nonnegative principal minors \cite{Mac, Boro, Bru18}, where $\II$ is the identity matrix on $\Xi$. In that case, one actually has a closed-form expression for the likelihood
\begin{equation} \label{eq:DPP-L}
\P(X=A)=\frac{\det[\mathbf{L}_A]}{\det[\II+\mathbf{L}]}, \quad \forall A\subseteq \Xi.
\end{equation}
\begin{definition}[$L$-ensemble]
  \label{d:L-ensembles}
  An $L$-ensemble with kernel $\LL$ is a DPP satisfying \eqref{eq:DPP-L}.
\end{definition}

\subsection{Sampling from DPPs: the HKPV meta-algorithm}
\label{ss:HKPV}
Whether DPPs are used as to extract summaries \cite{KuTa12}, select features \cite{BeBaCh18Sub}, or recommend baskets \cite{GBDK19}, sampling algorithms are needed. Sampling DPPs has indeed attracted considerable attention, from the original HKPV algorithm \cite{HKPV06} and its variants, see Section~\ref{ss:scalingUp}, to randomized numerical algebra \cite{Pou20} and related coupling constructions \cite{DeLaGa20}, or MCMC approximate samplers \cite{Kan13,LiJeSr15,ReKa15,AnGhRe16,LiJeSr16,GaBaVa17,HeSa19}. While a handful of exotic DPPs are amenable to computationally cheap adhoc approaches (e.g., uniform spanning trees \cite{PrWi98JoA}), most exact samplers are related to the original HKPV \cite{HKPV06}, investigated for finite $\Xi$ in \cite{KuTa12,Gil14}. An instance of the HKPV algorithm is given in Figure~\ref{f:HKPV}. 
In particular, a careful implementation of HKPV \cite[Section 2.4.4]{Gil14} has expected cost $\mathcal{O}(N^\omega + N\tau^2)$ time, where $\tau=\text{Tr}(\KK)=\mathbb{E}\vert X\vert$ acts as a sort of intrinsic dimension, to which HKPV effectively reduces the original dimension $N=\vert \Xi\vert$. Still, the bottleneck is usually the $\mathcal{O}(N^\omega)$ spectral decomposition of $\KK$.

%

\begin{figure}
  \begin{algorithm}{$\Algo{\cal{HKPV}}\big(\KK)$}
 \Aitem Perform spectral decomposition $\KK=\sum_{i=1}^{N} \la_i \phi_i\phi_i^{T}$.
 \Aitem Draw $N$ independent Bernoulli $B_i\sim \text{Ber}(\la_i)$. Set $k\setto \sum_{i=1}^N B_i$.
 \Aitem Initialise the kernel $\HH\setto\sum_{i \in I} B_i\phi_i \phi_i^{T}$ and the set $S\setto\emptyset$.
  \Aitem \For $i = 1,\dots,k$,
  \Aitem \mt Sample $\eta$ from $\Xi \setminus S$ with $\P(\eta = j) \propto \HH_{jj}$.
  \Aitem \mt Update $S \setto S \cup \{\eta\}$.
 \Aitem \mt Update $\HH \setto \HH - (\HH_{\eta\eta})^{-1} [\HH_{\Xi\eta} \HH_{\Xi\eta}^T] $.
  \Aitem \Return $S$.
  \end{algorithm}
\caption{
\label{f:HKPV}
The HKPV algorithm. $\KK_{\Xi\eta}$ stands for the $\eta$th column of $\KK$.
}
\end{figure}

In its dependence on spectral geometry, HKPV and its variants are primarily geared towards symmetric (or at least, Hermitian) kernels. At a high level, it can be viewed as a randomly pivoted Cholesky factorization \cite{Pou20}. There has been recent progress in extending this approach to $LU$ decompositions, yielding a $\mathcal{O}(N^3)$ sampler capable of addressing non-symmetric kernels \cite[Algorithms 1 and 4]{Pou20}, henceforth called the \emph{LU-based sampler}. Both the \emph{LU}-based sampler and HKPV become prohibitively expensive as $N$ grows, even disregarding storage constraints.

%
%

\subsection{Scaling up HKPV to large universal sets}
\label{ss:scalingUp}

A lot of work has gone into bypassing the cost of the spectral decomposition of $\KK$ in HKPV when $\KK$ is real symmetric, either through exploiting low-rank kernels \cite{KuTa11, GiKuTa12,AKFT13,AFAT13}, or by carefully downsampling the universal set \cite{Der19,DeCaVa19}. We focus here on low-rank kernels, as it is relevant in the context of our proposed method.
When the kernel is real symmetric and
\begin{equation}
  \label{e:nystrom}
  \KK=\BB^T\BB, \quad\text{where $\BB$ is $D\times N$ and $D\ll N$,}
\end{equation}
Kulesza and Taskar \cite{KuTa12} indeed show how the computational burden in HKPV can be kept down to the eigendecomposition of the $D\times D$ matrix $\BB\BB^T$. They actually start from a decomposition of $\LL=(\II-\KK)^{-1}\KK$, but the extension to $\KK$ is straightforward. Neglecting for now the cost of obtaining the decomposition \eqref{e:nystrom}, this yields a $\mathcal{O}(ND^2\tau)$ algorithm \cite[Section 2.4.4]{Gil14}, known as \emph{dual} HKPV.

 A popular decomposition like \eqref{e:nystrom} for DPPs is Nyström's \cite{WiSe01,AFAT13}. It consists in selecting a subset $Z\subset \Xi$ of cardinality $D$, and setting $\BB=\mathbf{S} \KK_{Z\Xi}$, with $\KK_{Z\Xi}$ the $D\times N$ submatrix of $\KK$ corresponding to the rows indexed by $Z$ and all columns, and $\mathbf S$ the square root of the pseudo-inverse of $\KK_Z$. Dual HKPV with Nyström decomposition thus remains a $\mathcal{O}(ND^2\tau)$ algorithm \cite{AFAT13}. In practice, the choice of $D$ and $Z$ for kernel machines is the topic of a rich literature; see \cite{GiMa16, Cal17} and pointers therein. One example approach with strong theoretical backing in kernel regression is to set $D$ sufficiently large compared to the trace of $\KK$ and sample $Z$ without replacement from a multinomial distribution, with weights given by so-called \emph{approximate ridge leverage scores}, computable in time linear in $N$ \cite{AlMa15}.

 One important limitation of scalable approaches to HKPV is that all work so far has focused on symmetric kernels $\KK$ with eigenvalues in $[0,1)$, usually by parametrizing a positive semidefinite symmetric $\LL$, which implicitly defines $\KK=(\II+\LL)^{-1}\LL$. But investigation on learning nonsymmetric kernels has started, since they offer significantly more modelling power \cite{Bru18,GBDK19}. In recommendation systems, for instance, allowing the signs of $\KK_{ij}$ and $\KK_{ji}$ to differ favours the co-occurrence of items $i$ and $j$ in DPP samples. Furthermore, many DPPs used as subsampling algorithms \cite{BeBaCh18Sub,BeBaCh19} have projection kernels, i.e. $\KK$ has eigenvalues in $\{0,1\}$, thus not fitting the requirement that the spectrum of $\KK$ lie in $[0,1)$. In this paper, we investigate a new scalable way to sample certain observables of DPPs called linear statistics, where neither symmetry nor the eigenvalues of $\KK$ play a role.

\section{The Laplace transform and sampling}
\label{s:laplace}
We refer to \cite[Chapter 5]{Kal06} and \cite{Doe74} for general references on Laplace transforms in probability and analysis, respectively. The Laplace transform of a non-negative random variable $Y$ is the function given, for $s \ge 0$, by the formula $\el_Y(s)=\E[e^{-sY}]$. The restriction of non-negativity on $Y$ and $s$ are for convergence purposes in the most general setting. If a real-valued random variable $Y$ has sufficiently light tails, then
$\el_Y(s)$ is well-defined for all complex numbers $s$. The fact that the domain of the Laplace transform can be extended to complex numbers will, in fact, be of crucial importance for our algorithmic approach. Finally, under very general conditions, the Laplace transform of a random variable uniquely identifies its distribution. The following will come in handy shortly.
\begin{example}
  \label{e:bernoulli}
  Let $Z_i \sim \mathrm{Ber}(p_i)$ be independent. Then $\el_{Z_1+\dots+Z_k}(s)=\prod_{i=1}^k (1-(1-e^{-s})p_i))$.
\end{example}

\subsection{The Laplace transform of linear statistics of a DPP}
\label{ss:laplace_transform_of_linear_statistic}
We provide here a closed form expression for the Laplace transform of a linear statistic of a DPP. Similar expressions for DPPs on more general sets are known, involving \textit{Fredholm determinants} (see, e.g., \cite{ST}). 
In the setting of most crucial interest in ML, the universal set $\Xi$ is finite, and we contribute here a much simpler result on the Laplace transform of linear statistics of finite DPPs. This has two advantages over the classical reference \cite{ST}.
First, all relevant quantities are expressed here in terms of usual determinants, which lets us use scalability techniques from the kernel machine literature. Second, our result is applicable to a much more general class of kernels and linear statistics than \cite{ST}, including the nonsymmetric kernels of \cite{Bru18,GBDK19}.
We state this as:

\begin{theorem} \label{thm:Laplace}
Let $X\sim\text{DPP}(\KK)$. We only assume that the probability measure on the subsets of $\Xi$ that satisfies \eqref{eq:DPP-def} is well-defined; in particular no further assumptions on $\KK$ are made vis-a-vis symmetry or otherwise. Let also $\Psi:\Xi \mapsto \C$. Then, for any $s \in \C$, the Laplace transform of the linear statistic $\L(\Psi):=\sum_{x\in X} \Psi(x)$ satisfies
\begin{equation} \label{eq:Laplace}
\el_{\L(\Psi)}(s) = \det[\II-\DD_{\Psi}\KK], \quad \text{where } \DD_\Psi=\mathrm{Diag}[(1-\exp(-s\Psi(i)))_{i \in \Xi}].
\end{equation}

\end{theorem}

One immediately recovers some known facts on DPPs. For instance, if $X\sim\text{DPP}(\KK)$, $A \subseteq \Xi$, and $1_A$ denotes the indicator of $A$, then a simple linear statistic is the number $N_A=\L(1_A)$ of points of $X$ that fall in $A$. Invoking Theorem \ref{thm:Laplace}, we get
$\el_{N_A}(s)= \prod_{\la_i\in \mathrm{Spec}(K_A)}(1-(1-e^{-s})\la_i)$.
We then recognize the Laplace transform of Example~\ref{e:bernoulli}, thus proving that $N_A$ is a sum of independent Bernoullis with parameters $\la_i \in \mathrm{Spec}(K_A)$. This is a non-trivial fact; see \cite{HKPV06} for a derivation using HKPV in the particular case of Hermitian kernels.


The proof of Theorem \ref{thm:Laplace} is deferred to Appendix~\ref{supp:s:proof_of_theorem}. By a continuity argument, we reduce to $L$-ensembles; see Definition~\ref{d:L-ensembles}. This is encapsulated in Lemma~\ref{lem:Laplace} below, which may be of independent interest and is proved in Appendix~\ref{supp:s:proof_of_lemma}. 

\begin{lemma} \label{lem:Laplace}
Let $X\sim \text{DPP}(\KK)$, in the sense that \eqref{eq:DPP-def} holds. Then there exists a sequence of DPPs $X_\eps$ on $\Xi$ with kernels $\KK_\eps$, indexed by the parameter $\eps \downarrow 0$, that are also $L$-ensembles (in the sense that there exist matrices $\LL_\eps$ such that \eqref{eq:DPP-L} holds), and $\KK_\eps \to \KK$ in the Frobenius norm.
\end{lemma}

\subsection{Numerically inverting a Laplace transform}
\label{ss:numerical_inversion}
In Section~\ref{ss:laplace_transform_of_linear_statistic}, we identified the law of $Y=\vert X\cap A\vert=\L(1_A)$ by looking at the closed-form Laplace transform of $Y$.
For more sophisticated Laplace transforms, this kind of identification is not possible.
However, as long as the Laplace transform can be evaluated pointwise, one can evaluate the distribution function $F(t) = \mathbb{P}(Y\leq t)$ numerically.
Indeed, it can be derived that for $s>0$,
$\int F(t) e^{-st} \d t = s^{-1} \el_Y(s)$,
so that, for $t\in\mathbb{R}$, one can approximate $F(t)$ by inverting a Laplace transform. Numerical inversion of Laplace transforms is a classical research topic; we refer to \cite{Kuh13} for a survey.
Most methods start from the so-called \emph{Bromwich} contour integral \cite[Equation (4)]{Kuh13}
\begin{equation}
  \label{e:bromwich}
  F(t) = \int_{\sigma + \mathrm{i}\mathbb{R}} s^{-1} \el_Y(s) e^{st}\d s\,,
\end{equation}
where $\sigma$ is any positive real number such that $\el_Y$ is analytic on $\text{Re}(s)\geq \sigma$.
Sophisticated choices for $\sigma$ and the discretization of \eqref{e:bromwich} have given several inversion algorithms, among which an algorithm by de Hoog, Knight, and Stokes (henceforth \textsc{deHoog}; \cite{HoKnSt82}).
\textsc{deHoog} forms a discrete sum approximating \eqref{e:bromwich} using the standard trapezoidal-rule with $E$ of nodes, but then actually builds a continued fraction expansion of the corresponding sum, and further uses acceleration techniques to provide a fast and accurate estimate of the evaluation of that expansion.
Neglecting the cost of evaluating the integrand, the resulting algorithm is polynomial in the number of evaluations $E$, which can usually be taken to be small \cite{Kuh13}; in the tens for all experiments in Section~\ref{s:experiments}.

\textsc{deHoog} has at least four advantages.
First, in the absence of a conclusive theoretical comparison, benchmarks and practice leads \cite{Kuh13} to recommend \textsc{deHoog} whenever $\el_Y$ is expensive to evaluate and $E$ needs to be small, which is our case.
Second, we have empirically found $\textsc{deHoog}$ to be robust to evaluation errors, which we will have to tolerate for large-scale examples where the kernel will be approximated.
Third, \textsc{deHoog} is available in the multi-precision arithmetic Python library \emph{mpmath} \cite{Joh18}.
Fourth, while the \emph{mpmath} implementation has a default rule of thumb to choose $\sigma$ depending on $t$, we can also keep $\sigma$ fixed for different values of $t$, as long as $\el_Y$ is analytic on $\text{Re}(s)\geq \sigma$. In that case, the nodes at which $\el_Y$ needs to be evaluated in \textsc{deHoog} do not depend on $t$. We can thus evaluate $F$ in \eqref{e:bromwich} at several values of $t$ using the same set of (costly) evaluations of $\el_Y$.

Once one has an approximate $F$, one has a convenient access to the distribution of $Y$, e.g., through its quantiles.
It is even possible to numerically solve $F(t)=U$ for $U\sim \mathcal{U}(0,1)$ to obtain an approximate sampler of $Y$ \cite{Rid09}.
On sampling with Laplace transforms, see also the rejection samplers of \cite{Dev81,Dev86} and the direct mixture-of-exponentials approximation of the PDF of $Y$ \cite{Wal17}.


\section{Our algorithm}
\label{s:algorithm}
For a DPP with kernel $\KK$ and a linear statistic $Y=\Lambda(\Psi)$ as in Theorem~\ref{thm:Laplace}, we propose to recover the CDF $F$ of $Y$ through de Hoog's inversion applied to \eqref{eq:Laplace}. The pseudocode in Figure~\ref{f:algorithm} summarizes how to evaluate $F$ at $T$ arbitrary points. Note how the loop can be parallelized, as we are perfoming $T$ independent numerical quadratures, with possibly different nodes. Additionally, the procedure does not put any constraint on $\KK$ and $\Psi$ other than defining a valid Laplace transform $\el_Y$ in \eqref{eq:Laplace}. With the inversion done, one can further compute approximate quantiles or sample $Y$; see Section~\ref{ss:numerical_inversion}.

\begin{figure}
  \begin{algorithm}{$\Algo{ApproxCDF}\big(\KK, \{t_1,\dots,t_T\}, \{\sigma_{t_1},\dots,\sigma_{t_T}\}, \{E_1,\dots,E_T\})$}
  \Aitem \For $t\in \{t_1,\dots,t_T\}$,
  \Aitem \mt Evaluate $\el_Y$ in \eqref{eq:Laplace} at the $E_t$ nodes on $\sigma_t + \mathrm{i} \mathbb{R}$ presribed by \textsc{deHoog}.\label{as:evaluation}
  \Aitem \mt \text{Apply \textsc{deHoog}'s quadrature to \eqref{e:bromwich}. Store the result in $\hat{F}_t$.}
  \Aitem
  \Return $(\hat{F}_{t})_{t\in\{t_1,\dots,t_T\}}$.
  \end{algorithm}
\caption{
\label{f:algorithm}
The pseudocode our approach. Keeping $\KK$ low-rank \eqref{e:nystrom} makes Step~\ref{as:evaluation} cost $\mathcal{O}(E_tND^2)$.}
\end{figure}

\begin{figure}
  \begin{algorithm}{$\Algo{ApproxCDFWithDiagonal}\big(\KK, \{t_1,\dots,t_T\}, \{\sigma_{t_1},\dots,\sigma_{t_T}\}, \{E_1,\dots,E_T\})$}
  \Aitem \For $t\in \{t_1,\dots,t_T\}$,
  \Aitem \mt For each quadrature node $s\in \sigma_t + \mathrm{i} \mathbb{R}$ prescribed by \textsc{deHoog},
  \Aitem \mtt Compute a low-rank approximation to $\DD_\Psi \KK$ in \eqref{eq:Laplace}.
  \Aitem \mtt Use that approximation to evaluate $\el_Y(s)$ as in \eqref{e:svdForLaplace}.
  \Aitem \mt \text{Apply \textsc{deHoog}'s quadrature to \eqref{e:bromwich}. Store the result in $\hat{F}_t$.}
  \Aitem
  \Return $(\hat{F}_{t})_{t\in\{t_1,\dots,t_T\}}$.
  \end{algorithm}
\caption{
\label{f:algorithm_variant}
The pseudocode of a variant of our approach, where we take the low-rank approximation into the loop. In practice, we use the approximate SVD of \cite{WLRT08}, which is quadratic in $N$.}
\end{figure}

\subsection{Comparison with the direct approach}
Say one is interested in the CDF of $Y=\Lambda(\Psi)$ at $T$ points $\{t_1,\dots,t_T\}$. Assume that HKPV can be applied, say $\KK$ is symmetric.
We need to compare the cost of our approach to HKPV.
Let us then use HKPV to sample $X_1,\dots,X_{M}$ i.i.d. from $\text{DPP}(\KK)$ at cost $\mathcal{O}(N^\omega + MN\tau^2)$, with $\tau=\text{Trace}(\KK)$; see Section~\ref{ss:HKPV}.
We then have $M$ i.i.d. samples $Y_i=\sum_{x\in X_i} \Psi(x)$, leading to the empirical CDF $\hat{F}_M(t) = \frac1M \sum_{i=1}^M 1_{Y_i\leq t}$.
The Dvoretzky-Kiefer-Wolfowitz inequality (DKW; \cite{Mas90}) further yields a $(1-\delta)$-confidence band of half-width $\sqrt{\log(2/\delta)/2M}$ around $\hat{F}_M$. In comparison, running our algorithm in Figure~\ref{f:algorithm} requires computing one $N\times N$ determinant per loop iteration and per quadrature node in the discretization of the Bromwich integral \eqref{e:bromwich}. Assuming that the number of nodes $E_i=E$ is constant for all $t_i$s for simplicity, this gives a $\mathcal{O}(TE N^\omega+TC)$ time complexity, where $C=\textsc{Poly}(E)$ is the complexity of running \textsc{deHoog}'s quadrature once the integrand in \eqref{e:bromwich} has been evaluated at $E$ nodes.
Furthermore, likely at the cost of some numerical accuracy due to not respecting the rule of thumb of $\emph{mpmath}$, we can also keep $\sigma_t=\sigma$ fixed across all values of $t$ and run \textsc{deHoog}; see Section~\ref{ss:numerical_inversion}. This allows to take Step~\ref{as:evaluation} out of the loop in Figure~\ref{f:algorithm}, taking the complexity down to $\mathcal{O}(E N^\omega + TC)$. We used that reduction in all the experiments of Section~\ref{s:experiments}.

Keeping in mind that $E$ is typically in the tens in practice for \textsc{deHoog}, the cost of our approach is comparable to HKPV whenever the $\mathcal O(N^\omega)$ cost of diagonalizing $\KK$ dominates the cost of HKPV. Thus, without any further structural assumption on $\KK$, our method only improves over HKPV in its wider applicability.
Our experiments in Section~\ref{s:experiments} further suggest that, for a similar cost, the result of \textsc{deHoog} is closer to the actual $F$ than the empirical cdf $\hat{F}_M$.
However, in all rigour, we would need a mathematical statement on the error of \textsc{deHoog}, in order to compare it to the DKW confidence band around $\hat F_M$. We could not locate such a mathematical statement in the numerical analysis literature.

\subsection{Scaling up to large universal sets}

Besides wide applicability, our approach shines in its scalability. First, we inherit low-rank arguments for HKPV. Indeed, whenever a decomposition $\KK = \BB^T \BB$ with $\BB$ a $D\times N$ matrix like \eqref{e:nystrom} can make HKPV more scalable, see Section~\ref{ss:scalingUp}, our method inherits the same scalability.
Indeed, using the \textit{spectrum trick} $\text{Spec}(\PP\QQ)\setminus \{0\} = \text{Spec}(\QQ\PP) \setminus \{0\}$,
evaluating $\el_Y$ in \eqref{eq:Laplace} boils down to evaluating
\begin{equation}
  \det[\II-\DD_{\Psi}\KK] = \det[\II-\DD_{\Psi}\BB^T\BB] = \det[\II-\BB \DD_\Psi\BB^T].
\label{e:nystromForLaplace}
\end{equation}
Computing \eqref{e:nystromForLaplace} takes $\mathcal{O}(ND^2)$ flops. This compares favourably with the expected cost $\mathcal{O}(MND^2\tau)$ of obtaining $M$ samples through HKPV with the same low-rank approximation.

Second, our algorithm can actually benefit from more widely applicable low-rank decompositions than HKPV, and thus provide a scalable alternative to the default $LU$-based sampler of \cite{Pou20} for generic DPPs.
For instance, even if $\KK$ is not symmetric, we can still  compute its SVD of rank $D$, $\KK \approx \UU \SSigma \VV^H$, with $\SSigma$ a $D\times D$ diagonal matrix with nonnegative entries, and $D\ll N$.
Then, using the same \textit{spectrum trick} as for \eqref{e:nystromForLaplace}, we can write $\el_Y$ in \eqref{eq:Laplace} as
\begin{equation}
  \det[\II-\DD_{\Psi}\KK] \approx \det[\II-\DD_{\Psi}\UU \SSigma \VV^H] = \det[\II-\SSigma^{1/2} \VV^H \DD_\Psi \UU \SSigma^{1/2}],
\label{e:svdForLaplace}
\end{equation}
which evaluates in $\mathcal O(ND^2)$ flops.
The bottleneck is thus the SVD. When $N$ is moderately large, we can use, e.g., the randomized SVD of \cite[Section 5.3]{WLRT08}, which requires $\mathcal{O}(N^2\log D + D^2 N\log N + ND^2)$ flops \cite[Remark 5.6]{WLRT08}. Overall, this provides a speedup over the $\mathcal O (N^3)$ exact $LU$-based sampler of \cite{Pou20}. Bigger speedups can naturally follow from other results on low-rank decompositions.

Third, we note that when evaluating $\el_Y(s)$ in \eqref{eq:Laplace}, it is the rank of $\DD_\Psi \KK$ that is the natural parameter, not the rank of $\KK$. It can well be in applications that the former is much smaller, since the rank of $\DD_\Psi$ is the support of $\Psi$. More generally, $\DD_\Psi$ could well have a lot of diagonal elements close to zero, as in the application of Section~\ref{s:experiments} to recommendation systems: $\Psi(i)$ is the price of item $i$ in a catalog $\Xi$, and the prices in the catalog concentrate towards zero. In Figure~\ref{f:algorithm_variant}, we give a variant of our algorithm that computes a low-rank approximation to $\DD_\Psi \KK$ for each new point $s$ where $\el_Y$ needs to be evaluated. In practice, we use again the approximate SVD of \cite{WLRT08}. On top of leveraging the lower rank of $\DD_\Psi \KK$, the random projection in \cite{WLRT08} now takes $\DD_\Psi$ into account, which intuitively should further improve accuracy compared to, e.g., using the ridge leverage scores of $\KK$ in Nyström \cite{AlMa15}.

\section{Experiments}
\label{s:experiments}

\paragraph{A synthetic symmetric kernel.} We take $N=10^3$, and draw a generic kernel of rank $100$ as $\KK = \sum_{i=1}^{100} \lambda_i v_i v_i^T$, where $\lambda_i=1/\sqrt{i}$ is a slowly decaying (deterministic) spectrum, and the $v_i$s are drawn i.i.d. from the Haar measure on $O_N(\mathbb{R})$. Since the kernel rank is only a tenth of $N$, low-rank approximations should intuitively be accurate. We consider an arbitrary linear statistic $\Psi(i)=\vert\cos i\vert$.

In Figure~\ref{f:results_symmetric_abscos}, we show the approximate CDFs obtained at $T=50$ equally spaced points. The blue baseline is the empirical CDF obtained from $10^4$ HKPV samples, using the implementation of \emph{DPPy} \cite{GBPV19}, with the DKW confidence band in shaded blue. In orange, we show the empirical CDF obtained from $100$ HKPV samples of DPP($\BB^T\BB$), where the $D\times N$ matrix $\BB$ is obtained by a Nyström approximation, using $D$ columns sampled without replacement using the approximate ridge leverage scores of \cite{AlMa15}, see Section~\ref{ss:scalingUp}. In the left panel of Figure~\ref{f:results_symmetric_abscos}, we take $D=\text{rk}(\KK) = 100$, while in the right panel we take $D=200$. In dashdotted green, we show the result of our algorithm from Figure~\ref{f:algorithm} applied to the Nyström kernel $\BB^T\BB$, with the same $E=41$ evaluations of $\el_Y$ to estimate $F(t)$ for all $t$s. In other words, the abscissa $\sigma$ of the Bromwich integral \eqref{e:bromwich} is kept fixed, to the smallest value proposed by \emph{mpmath} for the input $t$s. Finally, in dashed purple, we show the variant from Figure~\ref{f:algorithm_variant}, with the same $E=55$ evaluations of $\el_Y$ for all $t$s, thus needing $E$ SVDs in total. Using the \emph{mpmath} default $E=41$ nodes resulted in oscillatory behaviour in the right tail of $Y$. Since the resulting CDF is expected to be non-decreasing, oscillations are necessarily due to approximation error, and we slightly augmented the number of nodes to suppress the oscillations.


On the left panel, we observe that $D=\text{rk}(\KK)$ is not enough for Nyström to be a close approximation of the target CDF: for a third of the range, the Nyström confidence band does not intersect the confidence band on the true CDF. The result of our algorithm from Figure~\ref{f:algorithm} is a smoothed version of the empirical Nyström CDF. This makes our green curve more compatible with the profile that we guess in the blue band, but it is still as biased as Nyström. Among approximations of rank $D=\text{rk}(\KK)$, the clear winner is our algorithm from Figure~\ref{f:algorithm_variant}, with the purple and blue curves superimposed.

While future research on low-rank approximations of complex symmetric (not Hermitian) matrices may lower the cost of the purple curve and make it the default option, we want to further compare Nyström and the cheapest version of our algorithm in green. On the right panel of Figure~\ref{f:results_symmetric_abscos}, we increase the projected rank $D$ to twice the rank of $\KK$, which makes the Nyström confidence band include the blue confidence band. The Nyström empirical CDF remains however a poor approximation to the underlying CDF, compared to the perfect fit of our algorithms in green and purple. Thus, for a similar cost, the green curve is not only smoother than Nyström, but also more accurate.

\paragraph{Low-rank linear statistics.}
As noted in Section~\ref{s:algorithm}, our variant from Figure~\ref{f:algorithm_variant} can take advantage from $\DD_\Psi \KK$ being low-rank, even when $\KK$ is not. To see this, we switch the linear statistic to $\psi(i)=1/i$, so that many terms in $\DD_\Psi$ are close to zero. The resulting figures are very similar to Figure~\ref{f:results_symmetric_abscos}, and we defer them to Appendix~\ref{supp:s:low_rank_experiment}. As expected, Nyström and our algorithm from Figure~\ref{f:algorithm} suffer from lowering $D$ to $\text{rk}(\KK)/2 = 50$, while the purple curve of our variant from Figure~\ref{f:algorithm_variant} stays close to the blue baseline, although not in the blue DKW confidence band.

\begin{figure}
\centering
\subfigure[$ D = \text{rk}(\KK)$]{
\includegraphics[width=\twofigminus]{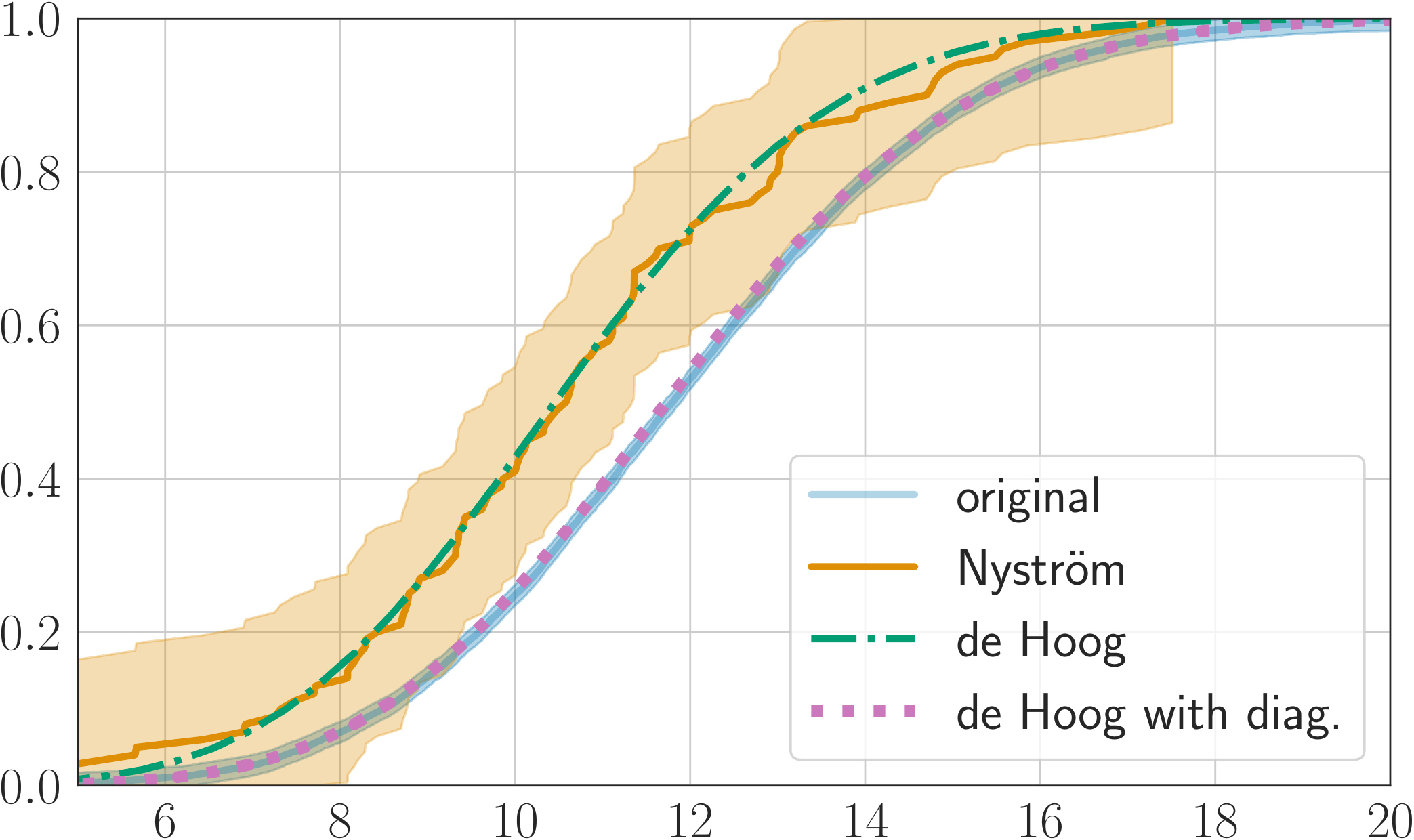}
}
\subfigure[$ D = 2\times \text{rk}(\KK)$]{
\includegraphics[width=\twofigminus]{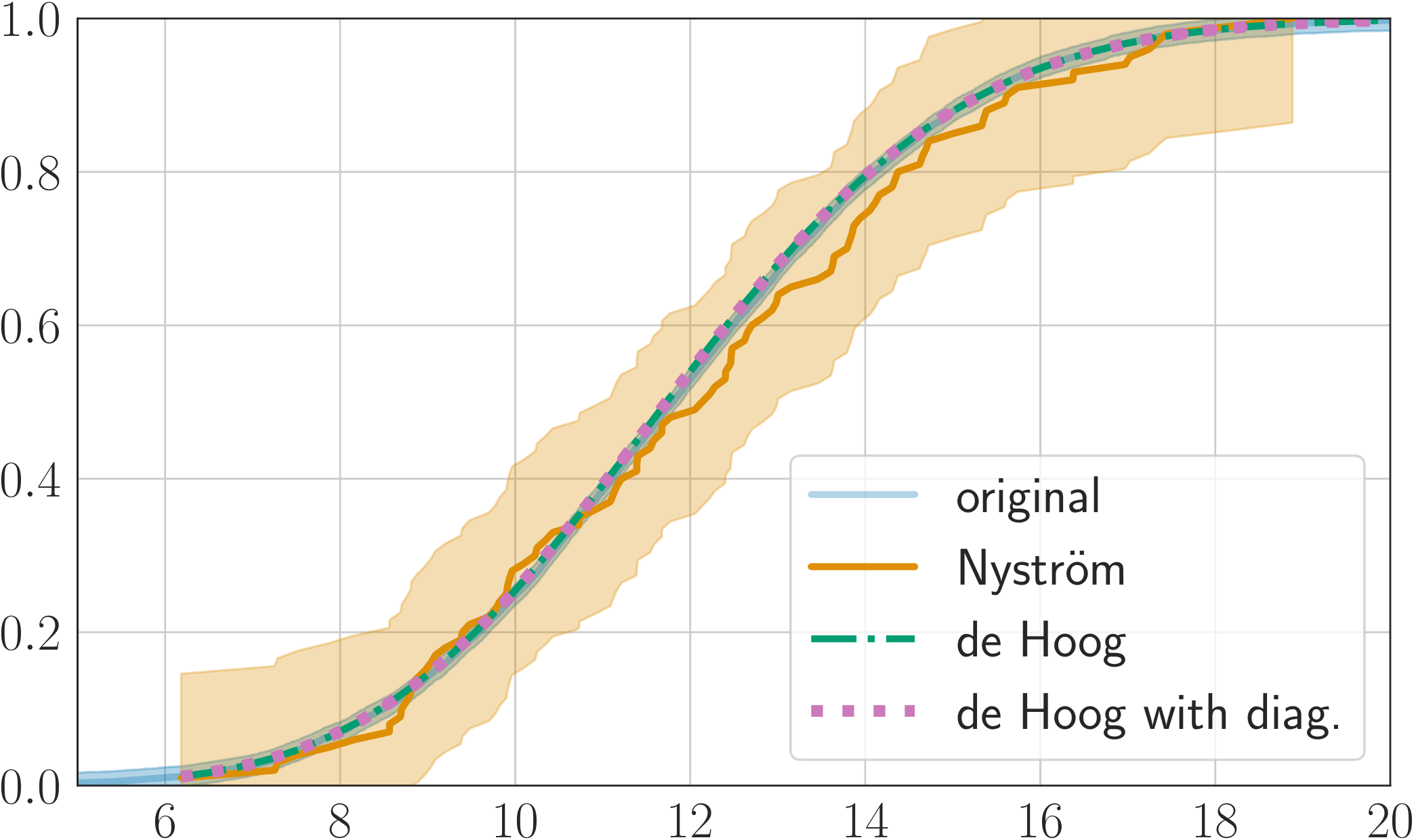}
}
\caption{Results of a synthetic experiment on symmetric kernels, with $N=10^3$ and $\Psi = \vert\cos(\cdot)\vert$.
\label{f:results_symmetric_abscos}
}
\end{figure}

To conclude the synthetic experiments, we always recommend our algorithm in Figure~\ref{f:algorithm} over Nyström to approximate a CDF, and we confirm that the variant in Figure~\ref{f:algorithm_variant} has the potential to further take down the rank of the approximation, although two aspects call for further investigation: the cost of the many SVDs and the best way to avoid oscillatory behaviour of the estimated CDF in the tail. We observed very similar results on synthetic nonsymmetric kernels (not shown).

\paragraph{A non-symmetric kernel for a recommendation system.}
We borrow a setting from \cite{GBDK19}, where the authors learn a nonsymmetric kernel from a large UK retail dataset, consisting in a list of 20\,728 orders --an order is a set of items-- from a catalog of around 4000 items. Samples from the learned DPP can thus be seen as candidate orders, and the DPP is ultimately used in tasks such as recommendations for basket completion. We took all parameters as in \cite{GBDK19} and use the \emph{PyTorch} \href{https://github.com/cgartrel/nonsymmetric-DPP-learning}{code} they provide for preprocessing the dataset and learning $\LL$. We obtain an $N\times N$ $L$-ensemble kernel $\LL$ of rank less than $100$ with $N=3941$ items. In particular, the learned kernel is constrained to have rank less than $100$. We then compute $\KK = (\II+\LL)^{-1}\LL$, see Section~\ref{s:determinantal}. The resulting nonsymmetric $\KK$ encodes both negative and positive correlations, in the sense that $\KK_{i,j}\KK_{j,i}$ can be of any sign; see our Figure~\ref{f:uk_results:excess} and \cite{GBDK19}. However, sampling from DPP($\KK$) is impractical: only the $LU$-based sampler of \cite{Pou20} applies, and our Python implementation takes 70 seconds for a single DPP($\KK$) sample on a modern laptop.

Assuming we are only interested in a linear statistic of the DPP, say the total price of the items in the basket represented by a DPP sample, we can apply the variant of our algorithm in Figure~\ref{f:algorithm_variant}. Using again the same $E=41$ quadrature nodes for all $T=100$ price values, and $D=100$, we obtain the approximate CDF in Figure~\ref{f:uk_results:cdf} in about 60 seconds, less than the time required for a single sample of the $LU$-based sampler. For comparison, we show in blue the empirical CDF obtained from $100$ samples of the $LU$-based sampler, obtained in about 2 hours. The range of the linear statistic is cut to 200, since a handful of very expensive items make the right tail very long, but we observed no oscillatory behaviour this time. Our algorithms thus unlock the exploration of generic DPP models.

\begin{figure}
\centering
\subfigure[$\KK_{ij}\KK_{ji}$ for $1\leq i,j\leq 50$]{
\includegraphics[width=\twofigMINUS]{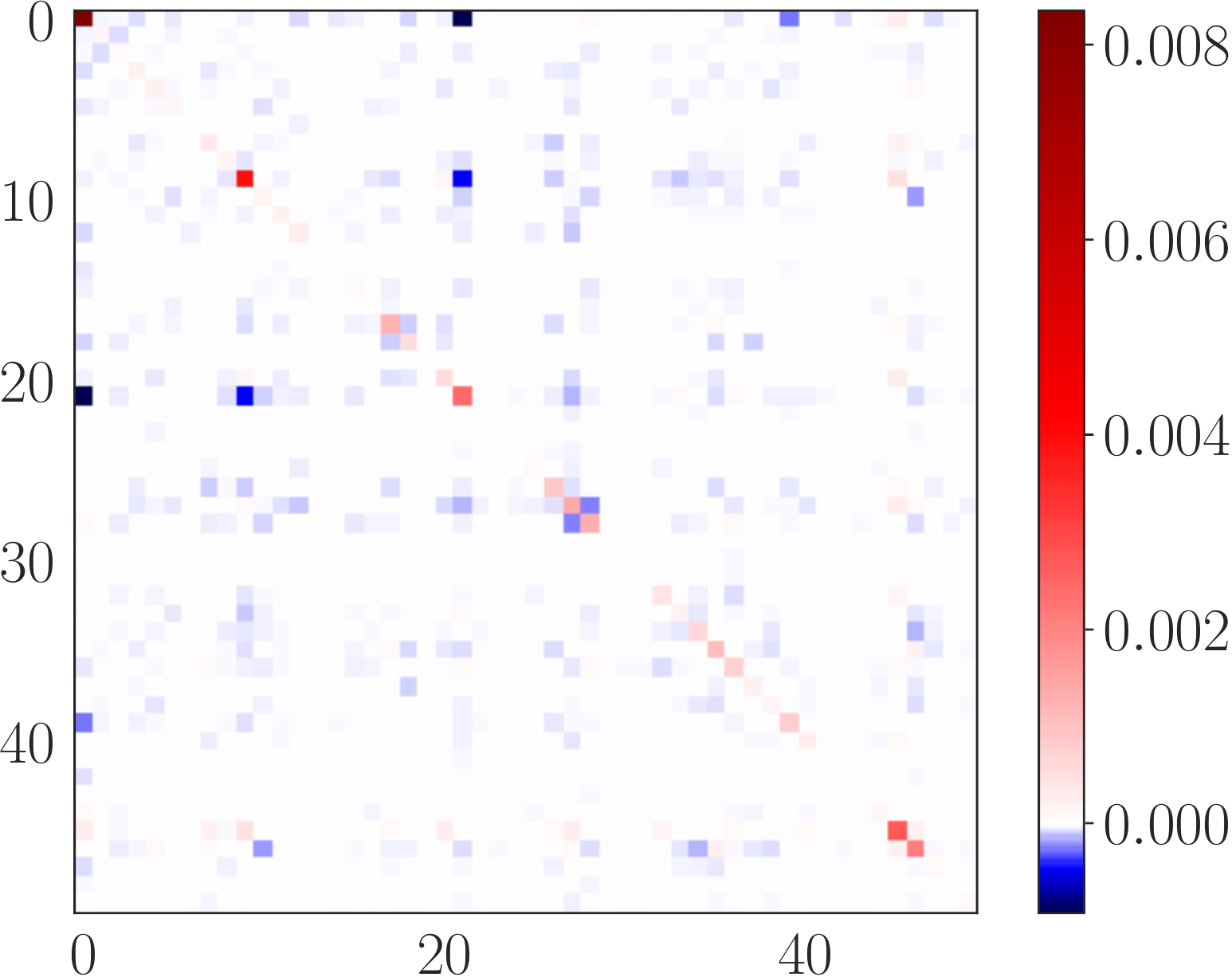}
\label{f:uk_results:excess}
}
\qquad
\subfigure[$D=100$]{
\includegraphics[width=\twofigminus]{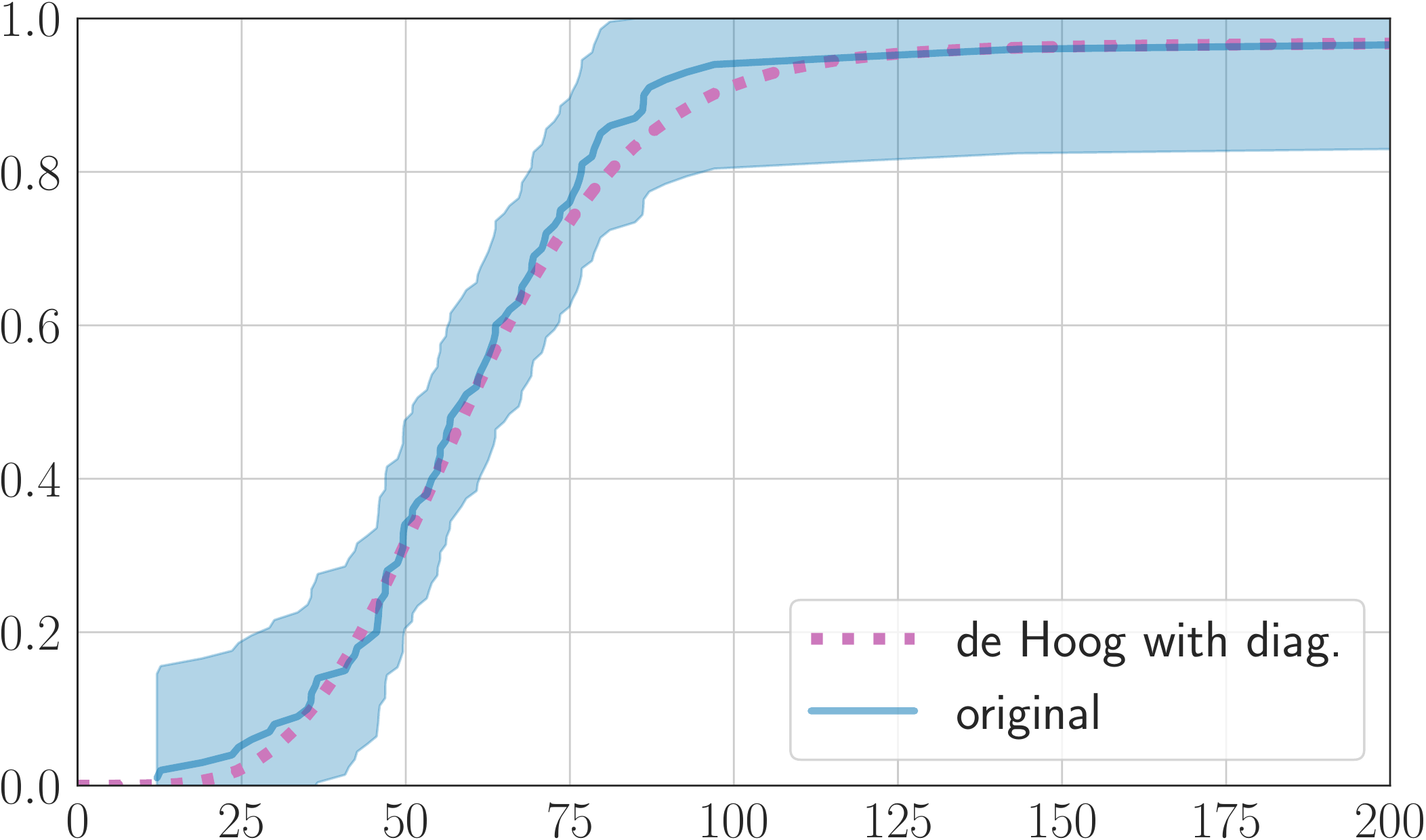}
\label{f:uk_results:cdf}
}
\caption{Results of the UK retail experiment with $N=3941$ and $\Psi(i)$ the price of item $i$.
\label{f:uk_results}
}
\end{figure}

\section{Discussion}
\label{s:discussion}
In terms of methodological flavour, our approach provides a bridge between  numerical analysis and probabilistic models. Natural avenues for further investigation include understanding the error of numerical inversion of the Laplace transform (e.g., de Hoog's method), for which, to our knowledge, there are no theoretical guarantees in the numerical analysis literature. Another natural direction would be extending our approach to DPPs on the continuum, which have attracted recent interest as spatial statistical models \cite{LaMoRu15}, or as sampling tools for kernel quadrature \cite{BeBaCh19}. A further natural question would be to extend this approach to other probabilistic models that are of interest in ML, beyond the particular setting of DPPs. For starters, conditioning a DPP to contain exactly $k$ points leads to the popular $k$-DPPs \cite{KuTa12}, which are mixtures of DPPs. As such, their Laplace transform is a linear combination of (many) determinants. Truncating the mixture to a tractable number of components with big weights should naturally extend our approach.

\section*{Acknowledgments}
RB acknowledges support from ERC grant \textsc{Blackjack} (ERC-2019-STG-851866). SG acknowledges support from MOE grant R-146-000-250-133.

\bibliographystyle{unsrt}
\bibliography{Sampling,zonotope,guillaume-thesis}
\appendix

\section{Proof of Theorem \ref{thm:Laplace}}
\label{supp:s:proof_of_theorem}

\begin{proof}[Proof of Theorem \ref{thm:Laplace}]
We begin by invoking Lemma \ref{lem:Laplace}, which leads to the fact that any DPP kernel $\KK$ can be approximated by DPP kernels $\{\KK_\eps\}_{\eps \downarrow 0}$, possibly along a sequence, such that the corresponding $L$-ensembles exist. This is equivalent to $(\II - \KK_\eps)$ being invertible, so that the matrices $\LL_\eps:=(\II - \KK_\eps)^{-1}\KK_\eps $ are well defined. In the case of a symmetric kernel $\KK$, for example, such an approximation can be obtained simply by thresholding the spectral decomposition of $\KK$ from the above at $(1-\eps)$, so that $\text{Spec}(\KK_\eps) \subset [0,1-\eps]$, and $\KK_\eps \to \KK$ in Frobenius norm as $\eps \downarrow 0$ . However, such arguments are crucially dependent on the symmetry of the kernel. Given the general scope of the present theorem, which aims to establish \eqref{eq:Laplace} as soon as the determinantal formulae \eqref{eq:DPP-def} holds for all the containment probabilities, this approximation requires more delicate consideration, and its existence is established in complete generality in Lemma \ref{lem:Laplace}. Since both the left and right hand sides of \eqref{eq:Laplace} are continuous in the kernel $\KK$, it  suffices therefore to establish \eqref{eq:Laplace} for kernels with well-defined $L$-ensembles: we may then invoke \eqref{eq:Laplace} for the kernels $\KK_\eps$ and subsequently let $\eps \downarrow 0$, possibly along a sequence.


In view of the above discussion, for the rest of the proof we confine ourselves to the situation where the DPP kernel $\KK$ corresponds to a well-defined $L$-ensemble of kernel $\LL=(\II-\KK)^{-1}\KK$, which we will exploit as an analytical tool. To this end, we first observe that if the realisation of the DPP $X$ equals a particular subset $A \subseteq \Xi$, the observed value of the linear statistic $\L(\Psi)$ is given by $\sum_{i \in A} \Psi(i)$. On the other hand, the probability of this event is given by $\frac{\det[\LL_A]}{\det[\II+\LL]}$; see \eqref{eq:DPP-L}. Together, these two facts imply that for any $t \in \C$, we have
\begin{align*}
\el_{\L(\Psi)}(t) = \E[\exp(-t \L(\Psi))] = &\sum_{A \subseteq \Xi} \exp(-t\sum_{i \in A}\Psi(i)) \cdot \frac{\det[\LL_A]}{\det[\II+\LL]} \\ = &\sum_{A \subseteq \Xi} \l(\prod_{i \in A}\exp(-t\Psi(i)) \r) \cdot \frac{\det[\LL_A]}{\det[\II+\LL]}.
\end{align*}
However, setting $\D_\Psi$ to be the diagonal matrix $\D_\Psi=\mathrm{Diag}[(\exp(-t\Psi(i)))_{i \in \Xi}]$, we note that $\prod_{i \in A}\exp(-t\Psi(i)) = \det[(\D_\Psi)_A]$, so we may write
\[  \l(\prod_{i \in A}\exp(-t\Psi(i)) \r) \cdot \det[\LL_A] = \det[(\D_\Psi)_A] \cdot \det[\LL_A] =  \det[(\D_\Psi)_A \LL_A]. \] Since $\D_\Psi$ is a diagonal matrix, we additionally have $(\D_\Psi)_A \LL_A = (D_\Psi \LL)_A$.

Combining all of the above, we may deduce that
\[\el_{\L(\Psi)}(t) = \sum_{A \subseteq \Xi}   \frac{\det[(\D_\Psi \LL)_A]}{\det[\II+\LL]} = \frac{ \sum_{A \subseteq \Xi} \det[(\D_\Psi \LL)_A]}{\det[\II+\LL]}.  \]
But, for any $\Xi \times \Xi$ matrix $\mathbf{M}$, we have $  \sum_{A \subseteq \Xi} \det[\mathbf{M}_A]=\det[\II+\mathbf{M}]$. Applying this to $\mathbf{M}=\D_\Psi \LL$, it enables us to further deduce that
\begin{align*}
\el_{\L(\Psi)}(t)  = &\frac{\det[\II + \D_\Psi \LL]}{\det[\II + \LL]} \\
= &\frac{\det[\II +  \LL \D_\Psi]}{\det[\II + \LL]}  \quad \text{(since $\det[\II+\mathbf{A}\mathbf{B}]=\det[\II+\mathbf{B}\mathbf{A}]$)}  \\
= &\det\l[(\II+\LL)^{-1} + (\II+\LL)^{-1}  \LL \D_\Psi \r] \\
= &\det[(\II-\KK)+\KK \D_{\Psi}]  \quad \text{(using $\KK=(\II + \LL)^{-1}\LL$)} \\
= &\det[\II - \KK(\II - \D_\Psi)] \\
= &\det[\II - (\II - \D_\Psi)\KK] \quad \text{(since $\det[\II+\mathbf{A}\mathbf{B}]=\det[\II+\mathbf{B}\mathbf{A}]$)} \\
= &\det[\II - \DD_\Psi\KK]  \quad \text{(using the definition of $\DD_\Psi$ to write $\DD_\Psi=\II - \D_\Psi$)},
\end{align*}
as desired. In the above derivation, we have made use of the fact that $\det[\II+\mathbf{A}\mathbf{B}]=\det[\II+\mathbf{B}\mathbf{A}]$ for any two matrices $\mathbf{A}$ and $\mathbf{B}$ for which the relevant matrix products are well-defined. This follows from the well-known fact that, for any such matrices, we have $\text{Spec}(\mathbf{A}\mathbf{B}) \cup \{0\} = \text{Spec}(\mathbf{B}\mathbf{A}) \cup \{0\}$. This completes the proof.

\end{proof}

We now prove Lemma \ref{lem:Laplace} which is a necessary ingredient for the proof of Theorem \ref{thm:Laplace}.

\section{Proof of Lemma \ref{lem:Laplace}}
\label{supp:s:proof_of_lemma}

It is perhaps worthwhile to briefly discuss the context for the main ideas contained in the development of Lemma \ref{lem:Laplace}. Our main goal is to use the likelihood formulae \eqref{eq:DPP-L}, which are only defined under certain invertibility conditions on the kernel of the DPP, and then take limits. Accordingly, we need to define approximating kernels (so that limits can be taken) which are also meaningful from the DPP perspective (so that \eqref{eq:DPP-L} holds true). In principle, we could take any reasonable approximation of the original  kernel $\KK$ and try to establish that the equations given by \eqref{eq:DPP-L} form a likelihood - i.e., they are non-negative and sum up to 1 as $A$ varies over the subsets of $\Xi$. However, the non-negativity of the right hand side of \eqref{eq:DPP-L} for an approximating kernel $\KK$ can be non-trivial in general, particularly beyond the symmetric situation when the illuminating spectral geometry of non-negative definite matrices is no longer applicable.

This motivates us to take an indirect approach, by first defining the associated inclusion probabilities for the approximating kernel in the form of a random experiment. This ensures that the stochastic constraints on the relevant determinants are satisfied - albeit for the inclusion probabilities \eqref{eq:DPP-def}. But the inclusion probabilities yield the likelihood equations \eqref{eq:DPP-L} via inclusion-exclusion relations and determinant identities, as soon as the approximants satisfy the invertibility conditions which are easy to check.

\begin{proof}
We first observe that, in order for a DPP satisfying \eqref{eq:DPP-def} with kernel $\mathbf{M}$ to be an $L$-ensemble with kernel $\LL$, it is enough that the matrix $(\II-\mathbf{M})$ is invertible. Indeed, this condition would immediately allow us to define the corresponding matrix $\LL(\mathbf{M})=(\II - \mathbf{M})^{-1}\mathbf{M}$. By the inclusion-exclusion principle, the probabilities $\l(\P(Y = A)\r)_{A \subseteq \Xi}$ and $\l(\P(A \subseteq Y)\r)_{A \subseteq \Xi}$ are in an invertible linear relationship with each other. In particular, the deduction of the collection of equations \eqref{eq:DPP-L} from the collection of equations \eqref{eq:DPP-def}, as $A$ varies over the subsets of $\Xi$, involves deterministic algebraic identities involving linear combinations of determinants, that holds in complete generality without any extra assumptions, as soon as the matrix $\LL(\mathbf{M})$ as above is well-defined.

In view of the above discussion, in order to establish the present lemma, we need to devise random subsets $X_\eps$ that are DPPs (in the sense of \eqref{eq:DPP-def}) such that the corresponding kernels $\KK_\eps$ satisfy two conditions : first, the matrix $(\II - \KK_\eps)$ is invertible, and secondly, $\KK_\eps \to \KK$ as matrices in the Frobenius norm as $\eps \downarrow 0$, possibly along a subsequence.

To this end, we will first define $X_\eps$ in terms of a probability measure on subsets of $\Xi$, and show that it is indeed a DPP in the sense of \eqref{eq:DPP-def} for some kernel $\KK_\eps$. For any $\eps>0$, we define the process $X_\eps$ as follows. First, we obtain a realisation of the process $X$, which is a subset of $\Xi$. Then, we retain each element of this subset independently with probability $(1+\eps)^{-1}$. This gives us a random subset $X_\eps$ of $\Xi$. To show that $X_\eps$ as defined is indeed a DPP, we observe that, for any $A \subseteq \Xi$ we have
\begin{align*}
\P(A \subseteq X_\eps) = &\P(\{A \subseteq X \} \cap \{\text{each elements of $A$ is retained}  \}) \\
=  &\P(A \subseteq X ) \cdot  \P(\text{each elements of $A$ is retained}) \\ = &\det[\KK_A] \cdot (1+\eps)^{-|A|} \\ = &\det[((1+\eps)^{-1}\KK)_A].
\end{align*}
Thus, the random subsets $X_\eps$ of $\Xi$ are indeed DPPs in the sense of \eqref{eq:DPP-def} with kernels $\KK_\eps=(1+\eps)^{-1}\KK$.

Clearly, as $\eps \to 0$, the matrices $\KK_\eps$ converge in the Frobenius norm to the matrix $\KK$. This takes care of the approximation property.

For the invertibility property, we notice that for the matrix $(\II - \KK_\eps)=(\II - (1+\eps)^{-1}\KK)$ to be non-invertible, the matrix $\KK$ must have $(1+\eps)$ as an eigenvalue. But the $\Xi \times \Xi$ matrix $\KK$ has at most $|\Xi|$ eigenvalues, which means that apart from the possible exception of a finite number of values of $\eps$, the matrices $(\II - \KK_\eps)$ must be invertible.

This completes the proof.

\end{proof}

\section{An experiment on a low-rank linear statistic}
\label{supp:s:low_rank_experiment}
We give here the results of an experiment described in Section~\ref{s:experiments} of the main paper, using the same synthetic symmetric kernel as in Figure~\ref{f:results_symmetric_abscos}, but with $\Psi(i)=1/i$. The results are shown in Figure~\ref{f:results_symmetric_inverse} for $D=\text{rk}(\KK)=100$ (left) and $D= \text{rk}(\KK)/2=50$ (right).

\begin{figure}
\subfigure[$ D = \text{rk}(\KK)$]{
\includegraphics[width=\twofig]{"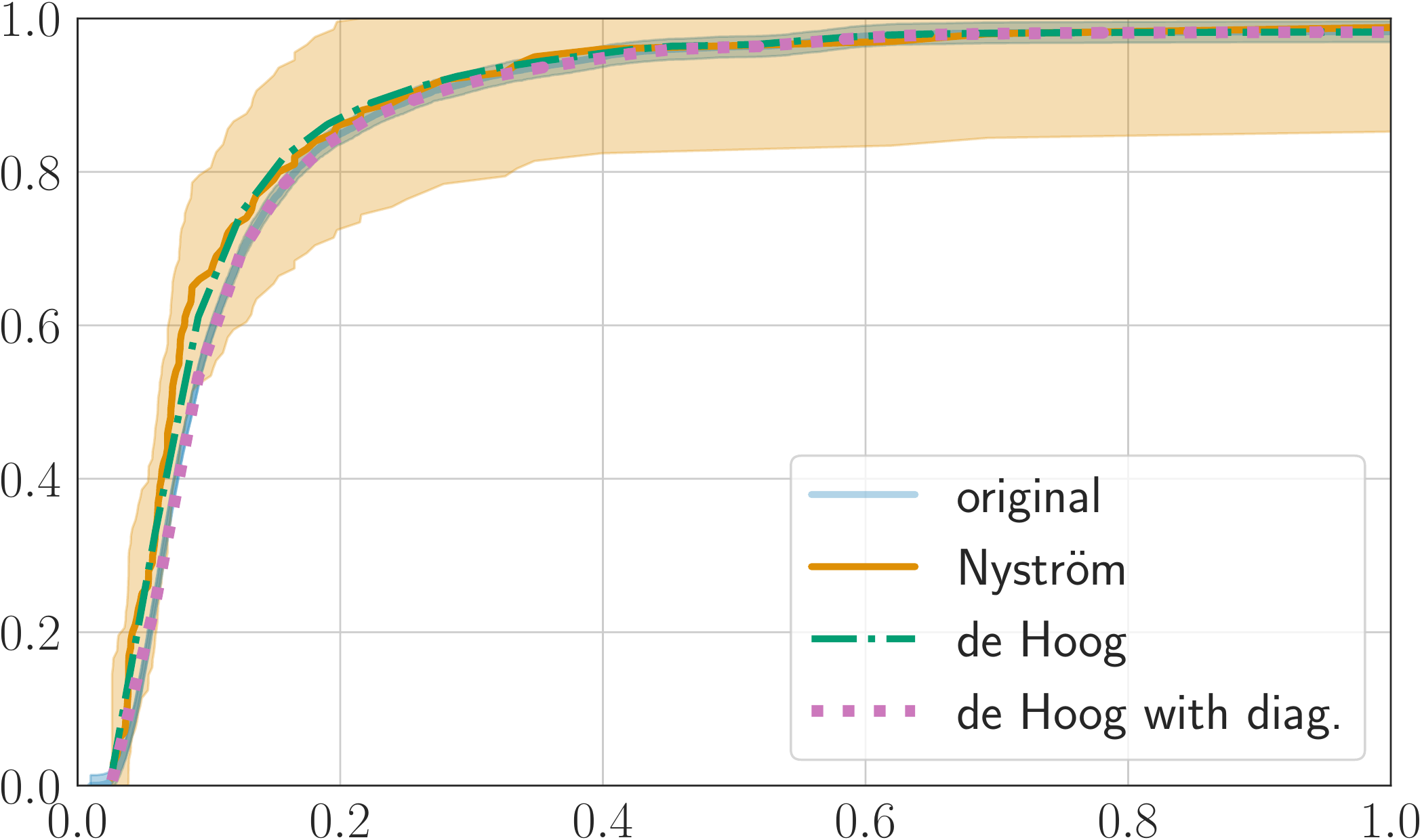}
}
\subfigure[$ D = 0.5 \times \text{rk}(\KK)$]{
\includegraphics[width=\twofig]{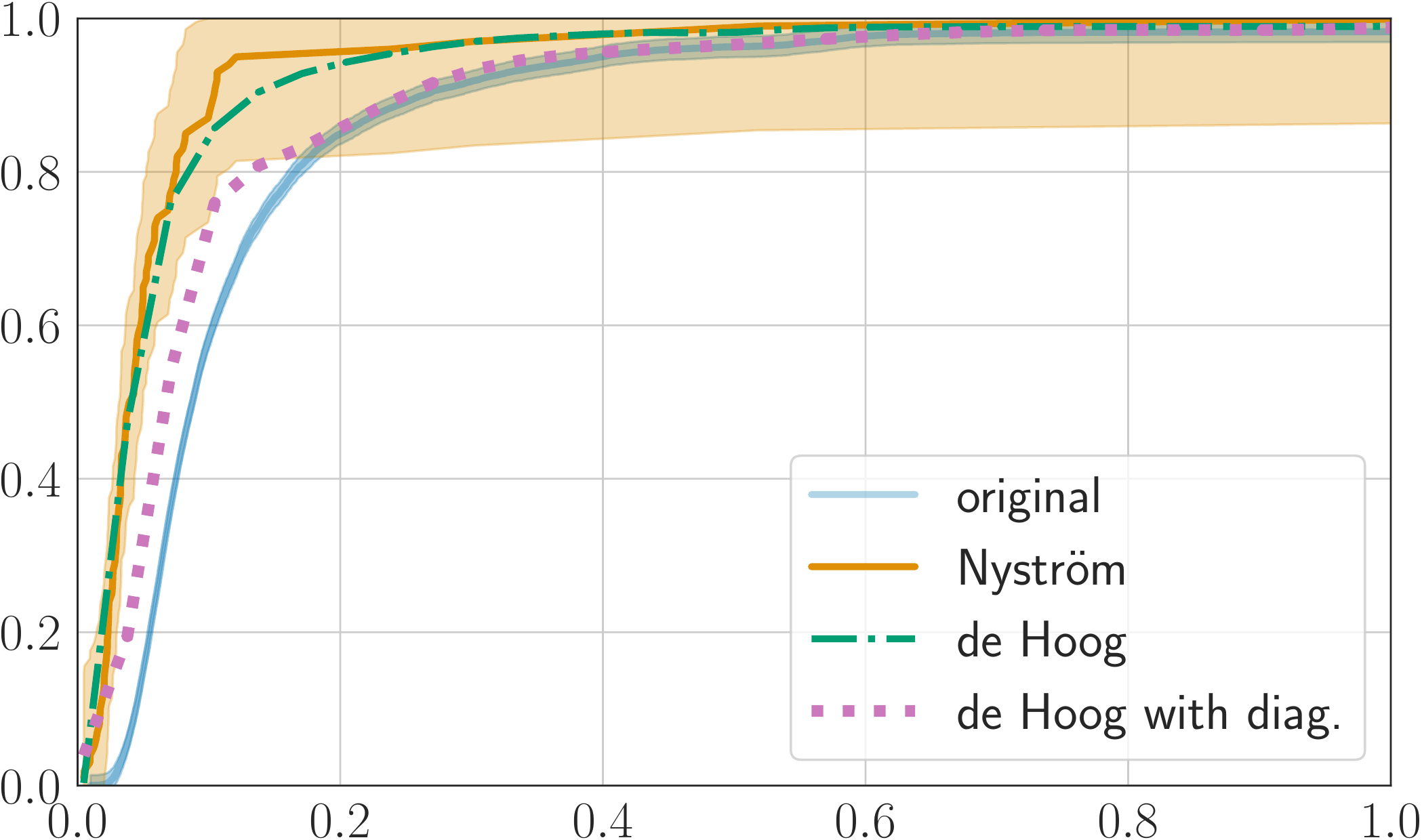}
}
\caption{Results of a synthetic experiment on symmetric kernels, with $N=10^3$ and $g(i) = 1/i$.
\label{f:results_symmetric_inverse}
}
\end{figure}

On the left panel, all approximations are in the same ballpark. Our algorithm from Figure~\ref{f:algorithm} is again a smoother version of the Nyström empirical CDF. And as expected, our variant from Figure~\ref{f:algorithm_variant} is the best fit among all approximations of rank $D$. One might even expect that our variant would fare well with $D$ smaller than the rank of $\KK$. The right panel confirms that: for $D=\text{rk}(\KK)/2$, Nyström loses accuracy, which is partly recovered by our green curve, with $E=41$ nodes and fixed $\sigma$. The purple curve remains closer to the blue baseline than the other two approximations. Surprisingly, we obtained the purple curve by taking $E=21$ nodes and fixed $\sigma$: we actually had to divide the default number of nodes of \emph{mpmath} by 2 to suppress a small oscillation appearing on the purple curve in the right tail of $Y$. Since CDFs are expected to be nondecreasing, an oscillation is necessarily due to approximation error. Increasing the number of nodes by up to $20$ nodes did not suppress the oscillation.

\end{document}